\documentclass{article}

\usepackage[preprint]{neurips_2026}

\usepackage[utf8]{inputenc} 
\usepackage[T1]{fontenc}    
\usepackage{hyperref}       
\usepackage{url}            
\usepackage{booktabs}       
\usepackage{amsfonts}       
\usepackage{nicefrac}       
\usepackage{microtype}      
\usepackage{xcolor}         
\usepackage{bbm}
\usepackage{fancyvrb}
\usepackage[commands,theorems]{AVT}
\usepackage[most]{tcolorbox}

\usepackage{algorithm}
\usepackage{algpseudocode}

\usepackage{subcaption}

\usepackage{amsmath}
\RequirePackage{amssymb}
\RequirePackage{amsthm}

\DeclareMathOperator{\KL}{\mathsf{KL}}

\DeclareMathOperator{\Id}{\mathrm{Id}}

\def\abs#1{\left|#1\right|}
\def\norm#1{\left\lVert#1\right\rVert}

\def\v#1{\boldsymbol{#1}}

\hypersetup{colorlinks,citecolor=blue}

\newcommand{\cN}{\mathcal{N}}

\newcommand{\cV}{\mathcal{V}}
\newcommand{\n}{\newline}

\newcommand{\ommit}{\begin{center}[\dots omitting lines for space\dots]\end{center}}
\newcommand{\MCB}{\mathrm{MCB}}
\newcommand{\DDPM}{\mathrm{DDPM}}

\newtcolorbox{promptbox}[2][]{
    enhanced,       
    colback=gray!5,
    colframe=black!80,
    title={#2},
    #1
}
\title{  Sampling from Flow Language Models via Marginal-Conditioned Bridges}

\author{%
  Iskander Azangulov \\
  Department of Statistics, University of Oxford\\
  \texttt{iskander.azangulov@spc.ox.ac.uk}
  \And
  Leo Zhang\\
    Department of Statistics, University of Oxford\\
  \texttt{leo.zhang@stx.ox.ac.uk}
}

\begin{document}

\maketitle

\begin{abstract}
\looseness=-1
Flow Language Models (FLMs) are a recently introduced class of language models which adapt continuous flow matching for one-hot encoded token sequences.  Their
denoisers have a special structure absent from generic continuous diffusion
models: each block of the denoising mean is a posterior marginal distribution
over the clean token at that position.  
Standard DDPM-style samplers collapse these marginals to a single conditional-mean endpoint and bridge toward this
simplex-valued point, which is generally not a valid one-hot sequence.  We argue
that the natural sampler for an FLM is instead posterior-predictive.  At each
reverse step, we sample a clean one-hot endpoint from the factorized posterior
defined by the FLM token marginals, and then sample the next continuous state
from the analytic Ornstein--Uhlenbeck bridge conditioned on that endpoint.  The
method is training-free, uses the same model evaluations as standard sampling,
and gives a principled interface for token-level decoding controls such as
temperature scaling and nucleus truncation.  We show that, under exact posterior
marginals, the endpoint approximation error is exactly the conditional
multi-information among token positions.  The induced one-step bridge kernel
preserves all token-wise posterior-predictive marginals and loses only the
residual cross-position dependence.  Finally, we prove a Girsanov path-space
comparison showing that the marginal-conditioned bridge has a no-larger
denoising-error term than the frozen conditional-mean bridge, with strict
improvement whenever intermediate coordinate-wise bridge observations reveal
additional information about the clean token.
Experiments with FLMs show that the sampler improves the quality--diversity
tradeoff.
\begin{center}
    Code is available at: \href{https://github.com/imbirik/mcb}{github.com/imbirik/mcb}.
\end{center}

\end{abstract}

\section{Introduction}

\looseness=-1
Autoregressive language models remain the dominant paradigm for language modeling, demonstrating impressive performance in recent years \citep{achiam2023gpt, team2023gemini}, however their left-to-right factorization makes sampling inherently sequential in the output length.  This limitation has motivated a growing family of
non-autoregressive language models that update many token positions in parallel,
including discrete diffusion models, discrete flow matching models, and
continuous diffusion or flow models for text~\citep{austin2023structureddenoisingdiffusionmodels, campbell2022continuoustimeframeworkdiscrete,lou2024discretediffusionmodelingestimating,sahoo2024simpleeffectivemaskeddiffusion, gat2024discreteflowmatching}.  This paper focuses on a recent member of
this family: Flow Language Models (FLMs), which apply continuous denoising or
flow modeling to one-hot token sequences
\citep{lee2026flow,chen2026langflow,roos2026categorical, potaptchik2026discrete}.

The one-hot representation gives FLMs a posterior structure that is absent from
generic continuous diffusion models.  Let
\(X_0\in\{e_1,\ldots,e_{|\mathcal V|}\}^L\) denote a clean one-hot sequence, and
let \(X_t\) be its Gaussian corrupted version.  The denoising mean
\[
    m_t(x):=\E[X_0\mid X_t=x]
\]
is not merely a vector-valued regression target.  Since each block
\(X_{0,\ell}\) is one-hot, the \((\ell,v)\)-coordinate of \(m_t(x)\) is exactly
a posterior token probability:
\[
    (m_t(x))_{\ell,v}
    =
    \P(X_{0,\ell}=e_v\mid X_t=x).
\]
Thus the output of an FLM can be interpreted as token-wise posterior marginals
over the clean sequence.

Most FLM sampling procedures inherit the standard continuous-diffusion
convention: the denoiser is treated as a point estimate, and the reverse update
bridges toward the conditional mean \(m_t(x)\).  This is a natural convention
for continuous data, where the denoising mean may itself be a plausible data
point.  In a one-hot language model, however, the same operation has a different
interpretation.  The mean \(m_t(x)\) usually lies in the interior of a product
of simplices; it is not a valid one-hot sequence, but a representation of
uncertainty over many possible clean sequences.  A bridge to this mean is
therefore a moment-matched approximation to a posterior-predictive transition,
rather than the transition obtained by sampling from the posterior represented
by the model.

This observation leads to a simple question: given the posterior marginals
provided by an FLM, what is the natural posterior-predictive reverse step?  For
two noise levels \(s<t\), the exact reverse transition can be written as
\[
    q(x_s\mid x_t)
    =
    \sum_{x_0\in\mathcal X_0}
    q(x_s\mid x_t,x_0)\,q(x_0\mid x_t),
\]
where
\[
    \mathcal X_0:=\{e_v:v\in\mathcal V\}^L
\]
is the set of valid one-hot sequences.  For the Ornstein--Uhlenbeck corruption
process considered in this paper, the bridge transition
\(q(x_s\mid x_t,x_0)\) is Gaussian and available in closed form.  The only
intractable object is the joint clean posterior \(q(x_0\mid x_t)\), which
couples all token positions.  An FLM, however, provides its token marginals.
This naturally defines the factorized endpoint posterior
\[
    q_t^I(x_0\mid x)
    :=
    \prod_{\ell=1}^L q_{t,\ell}(x_{0,\ell}\mid x),
\]
which is precisely the posterior family represented by the model.

We propose marginal-conditioned bridge (MCB) sampling, the posterior-predictive
sampler induced by this factorized endpoint posterior.  At each reverse step,
the sampler evaluates the FLM posterior marginals at the current continuous
state, independently samples one clean token at every position, and then samples
the next continuous state from the analytic OU bridge conditioned on both the
current state and the sampled one-hot endpoint.  Equivalently, the sampler
replaces a single bridge to the conditional-mean endpoint by a mixture of
bridges to valid clean sequences:
\[
    \sum_{x_0\in\mathcal X_0}
    q_{\mathrm{OU}}(x_s\mid x_t,x_0)
    \prod_{\ell=1}^L q_{t,\ell}(x_{0,\ell}\mid x_t).
\]
The method is training-free and uses the same FLM posterior-marginal evaluation
as the conditional-mean sampler.  Because the endpoint distribution is
categorical, the sampler also gives a direct non-autoregressive interface for
token-level decoding controls such as temperature scaling and nucleus
truncation.

Furthermore, we provide a theoretical analysis of the proposed MCB sampler.
Our analysis separates the statistical approximation from the sampling
procedure.  First, under exact posterior marginals, replacing the true joint
clean posterior by the product of its token marginals incurs exactly the
conditional multi-information among token positions.  With approximate
marginals, the error decomposes into this dependence term plus token-wise
marginal estimation errors.  Second, the induced bridge kernel preserves every
token-wise posterior-predictive transition marginal and loses only residual
cross-position dependence.  Third, we compare the continuous-time bridge
interpolations of marginal-conditioned sampling and the conditional-mean bridge.
Under standard Girsanov regularity assumptions, the marginal-conditioned bridge
has no larger path-space KL error to the true reverse process, with strict
improvement whenever the coordinate-wise bridge observation provides additional
information about the clean token.

Our contributions are:
\begin{enumerate}
    \item We introduce marginal-conditioned bridge sampling, a training-free
    posterior-predictive sampler for FLMs that samples valid one-hot clean
    endpoints from the factorized posterior represented by the FLM marginals and
    then applies the analytic OU bridge transition.

    \item We prove that the cost of factorizing the clean endpoint posterior is
    exactly the conditional multi-information among token positions, with an
    additional token-wise decomposition for approximate posterior marginals.

    \item We show that the one-step bridge kernel induced by the sampler
    preserves all token-wise posterior-predictive transition marginals and loses
    only cross-token dependence.

    \item We prove a Girsanov path-space comparison showing that, under exact
    posterior marginals and standard regularity assumptions, the
    marginal-conditioned bridge has no larger KL error than the
    conditional-mean bridge.

    \item We evaluate the sampler on FLMs trained on LM1B and OWT and show that it
    improves the quality--diversity tradeoff without retraining, while enabling
    endpoint-level temperature and nucleus decoding.
\end{enumerate}

\section{Related Work}

\paragraph{Diffusion, flow matching, and bridge samplers.}
Diffusion models generate samples by reversing a noising process, either through
discrete-time denoising transitions or continuous-time reverse SDEs
\citep{sohl2015deep,ho2020denoising,song2020score}.  Flow matching and
stochastic interpolants provide closely related continuous-time frameworks in
which a model learns a vector field or denoiser along a prescribed path between
noise and data \citep{lipman2022flow,albergo2022building,albergo2025stochastic}.
Our work uses the Ornstein--Uhlenbeck path because its pinned bridge transition
has a closed form.  This allows posterior information over clean endpoints to be
converted directly into an analytic reverse transition.

\paragraph{Discrete diffusion and flow models for language.}
A major line of work builds generative models directly on discrete state spaces.
Structured categorical diffusion models define noising and denoising kernels on
tokens, while continuous-time discrete denoising models formulate the process as
a Markov chain over discrete states~\citep{austin2023structureddenoisingdiffusionmodels,campbell2022continuoustimeframeworkdiscrete}.  Discrete flow matching
models similarly learn probability paths over categorical variables and have
shown promising results for language and code generation~\citep{lou2024discretediffusionmodelingestimating,sahoo2024simpleeffectivemaskeddiffusion,gat2024discreteflowmatching}.  These methods
typically rely on tractable token-wise or factorized reverse transitions.  Our
analysis makes the cost of such factorization explicit: the missing term is the conditional mutual-information among token positions.

\paragraph{Flow language models.}
Another line of work applies continuous diffusion and flow matching to language modeling by embedding tokens into Euclidean space.  Flow Language Models (FLMs) take a particularly direct approach by applying the continuous corruption process to one-hot token representations \citep{lee2026flow,chen2026langflow,roos2026categorical, potaptchik2026discrete}.  In
this representation, the conditional mean can be viewed as the factorized posterior  over clean tokens.
Our work focuses on analyzing and improving the sampling of such FLMs by exploiting this additional structure.
Closest to our work, \cite{sevriugovlogit} also samples clean endpoints from
factorized token marginals, but its stochastic inference step re-noises the
sampled endpoint using a fresh base sample rather than sampling the conditional
bridge; the authors note that this sampling-only
procedure suffers from low entropy and introduce a hybrid ODE-stochastic sampler to mitigate this, whereas our method uses the exact OU bridge conditioned on both the current state and the sampled endpoint which does not suffer from this issue.

\paragraph{Few-step generation and decoding controls.}
Many approaches to faster generation modify the training objective, distill a
many-step sampler into a few-step sampler, or learn a direct flow map.  By
contrast, our method changes only the inference procedure.  Because the bridge
endpoint is sampled from explicit categorical distributions, standard decoding
controls can be applied directly to the endpoint marginals.  For example,
temperature scaling changes the sharpness of the factorized endpoint posterior,
and nucleus sampling truncates its low-probability tail
\citep{holtzman2020curiouscaseneuraltext}.  In this sense, marginal-conditioned bridges give
FLMs a natural non-autoregressive analogue of token-level decoding.

\section{Background}

\subsection{Diffusion models}

Suppose we want to generate samples from a distribution \(\mu\) on
\(\R^d\).  Diffusion models achieve this by reversing a forward process that
progressively corrupts data into noise.  We consider the Ornstein--Uhlenbeck \citep{maller2009ornstein}
forward process
\begin{equation}
\label{eq:forward}
\begin{cases}
    dX_t=-X_t\,dt+\sqrt{2}\,dB_t,\qquad t\in[0,T],\\
    X_0\sim\mu,
\end{cases}
\end{equation}
where \(B_t\) is a \(d\)-dimensional Brownian motion.  Let $c_t:=e^{-t}, \sigma_t^2:=1-e^{-2t}$, we can express the forward transition kernel in the closed-form: $X_t\mid X_0
    \sim
    \cN(c_tX_0,\sigma_t^2\Id_d)$.
We write \(p_t\) for the marginal density of \(X_t\) and
\(s_t(x):=\nabla_x\log p_t(x)\) for its score.

Under standard regularity assumptions \citep{ANDERSON1982313}, the reverse
process \(Y_t:=X_{T-t}\) satisfies
\begin{equation}
\label{eq:backward}
\begin{cases}
    dY_t=\bigl(Y_t+2s_{T-t}(Y_t)\bigr)\,dt+\sqrt{2}\,dB'_t,\qquad t\in[0,T],\\
    Y_0\sim p_T,
\end{cases}
\end{equation}
where \(B'_t\) is another Brownian motion.  Thus, if we could initialize
\(Y_0\sim p_T\) and simulate \eqref{eq:backward} exactly, then
\(Y_T\sim\mu\).
In practice, we approximate the score function with a learned neural network, we initialize $Y_0$ at $\cN(0, \Id_d)$ and discretize \eqref{eq:backward} to draw samples approximating $\mu$.

Additionally, we note that the score can be written in terms of the conditional mean $m_t$ by Tweedie's formula
\begin{equation}
\label{eq:tweedie}
    s_t(x)
    =
    \frac{c_t m_t(x)-x}{\sigma_t^2} \qquad \text{ where } \qquad m_t(x):=\E[X_0\mid X_t=x]
\end{equation}
Therefore, the reverse drift can be parameterized by an estimate of
\(m_t(x)\).

\subsection{Flow Language Models}
\label{sec:flm}

Let \(\cV\) be a finite vocabulary, \(V:=|\cV|\), and let \(L\) be the sequence
length.  A sequence \(w=(w_1,\ldots,w_L)\in\cV^L\) is embedded as a one-hot
vector \(x(w)\in\R^{D}\), where \(D=LV\), with block
\[
    x(w)_\ell=e_{w_\ell}\in\R^V.
\]
Let \(\nu\) be a distribution on \(\cV^L\), and let \(\mu\) be its pushforward
under the one-hot embedding.  A Flow Language Model (FLM) applies the continuous
corruption process \eqref{eq:forward} to samples \(X_0=x(W)\), \(W\sim\nu\) and defines a generative model of sequence data from learning and sampling the reverse process.

The key property of the one-hot representation is that the denoising mean is a
collection of token posterior marginals.  Indeed,
\begin{equation}
\label{eq:flm_marginals}
    \bigl(m_t(x)\bigr)_{\ell,v}
    =
    \P(W_\ell=v\mid X_t=x)
    =
    \P(X_{0,\ell}=e_v\mid X_t=x).
\end{equation}
Thus an FLM can be trained as a posterior-marginal predictor
\[
    p_{\theta,t,\ell}(\cdot\mid x)\in\Delta_{V-1},
\]
where $\Delta_{V-1}$ denotes the $(V-1)$-simplex, via the cross-entropy objective:
\begin{equation}
\label{eq:flm_loss}
    \mathcal L(\theta)
    =
    \E_{W\sim\nu,\;Z\sim\cN(0,I_D),\;t\sim\lambda}
    \left[
        -\sum_{\ell=1}^L
        \alpha_t
        \log
        p_{\theta,t,\ell}
        \bigl(W_\ell\mid c_t x(W)+\sigma_t Z\bigr)
    \right],
\end{equation}
where \(\lambda\) is the training-time distribution over noise levels and
\(\alpha_t\) is a weighting function.  At the population optimum $\theta^*$, we recover the exact posterior marginals:
\[
    p_{\theta^*,t,\ell}(\cdot\mid x)
    =
    \P(W_\ell=\cdot\mid X_t=x).
\]
This additional posterior marginal structure to FLMs provides the structure that we  exploit in our proposed sampler.

\subsection{OU bridges}
\label{sec:bridge}

From \citep{Mazzolo_2017}, the OU process admits an analytic bridge transition kernel.  For \(0\le s<t\), conditioning
on \(X_0=x_0\) and \(X_t=x_t\) gives
\begin{equation}
\label{eq:bridge_marginal_clean}
    X_s\mid X_0=x_0,\;X_t=x_t
    \sim
    \cN
    \left(
        \frac{\sinh(t-s)}{\sinh t}x_0
        +
        \frac{\sinh s}{\sinh t}x_t,\,
        2\frac{\sinh s\,\sinh(t-s)}{\sinh t}\Id_d
    \right).
\end{equation}
Equivalently, the pinned OU bridge ending at \(x_t\) satisfies
\begin{equation}
\label{eq:bridge_sde}
    dX_s^{x_t}
    =
    \frac{x_t-X_s^{x_t}\cosh(t-s)}{\sinh(t-s)}\,ds
    +
    \sqrt{2}\,dB_s,
    \qquad s\in[0,t).
\end{equation}
This bridge formula is the basic transition kernel employed by our sampler.

The reverse process can be viewed as a mixture of such bridges.  Conditional on
\(Y_\tau=y\), equivalently \(X_{T-\tau}=y\), the remaining reverse path is a
mixture of OU bridges from \(y\) to a clean endpoint
\[
    X_0\sim \mathcal L(X_0\mid X_{T-\tau}=y).
\]
\looseness=-1
Standard DDPM-style sampling replaces this random endpoint by the conditional
mean \(m_{T-\tau}(y)\).  Our sampler instead samples a one-hot endpoint from the
factorized posterior represented by the FLM.

\subsection{Conditional-mean DDPM sampling}
\label{sec:disc}

Let $0=t_0<t_1<\cdots<t_K=T$ be a reverse-time grid and define \(\gamma_k:=t_{k+1}-t_k\).  At reverse time
\(t_k\), the corresponding forward noise level is $u_k:=T-t_k$.

As shown in \cite{potaptchik2025linearconvergencediffusionmodels}, the DDPM sampler~\cite{ho2020denoising} can be written as a bridge sampler that freezes an estimate of
the clean endpoint at the current grid point.  If \(\hat m_{u_k}(y_{t_k})\) is an
estimate of \(m_{u_k}(y_{t_k})=\E[X_0\mid X_{u_k}=y_{t_k}]\), then the sampler
uses the OU bridge from \(y_{t_k}\) at noise level \(u_k\) to the deterministic
endpoint \(\hat m_{u_k}(y_{t_k})\).

In continuous time, this corresponds on each interval
\([t_k,t_{k+1})\) to the frozen-endpoint drift
\begin{equation}
\label{eq:frozen_mean_drift}
    d\hat Y_t
    =
    \frac{
        \hat m_{u_k}(\hat Y_{t_k})
        -
        \hat Y_t\cosh(T-t)
    }{
        \sinh(T-t)
    }\,dt
    +
    \sqrt{2}\,dB'_t.
\end{equation}
Using \eqref{eq:bridge_marginal_clean}, this is equivalent to the usual DDPM scheme.  The corresponding Girsanov bound  has the
form
\begin{align}
\label{eq:girsanov}
&\KL\bigl(\mathrm{Law}(Y_{[0,T]})\,\Vert\,\mathrm{Law}(\hat Y_{[0,T]})\bigr)
\nonumber\\
&\qquad
\le
\KL\bigl(p_T\Vert\cN(0,\Id_d)\bigr)
+
\sum_k
\int_{t_k}^{t_{k+1}}
\frac{c_{T-t}^2}{\sigma_{T-t}^4}
\E
\left[
    \norm{
        m_{T-t}(X_{T-t})
        -
        \hat m_{u_k}(X_{u_k})
    }^2
\right]dt,
\end{align}
with equality under the usual absolute-continuity and Novikov conditions.  The
quantity inside the integral is the denoising error induced by freezing the
endpoint estimate over the interval.

\section{Main Results}
\label{sec:main_result}

We now introduce the marginal-conditioned bridge (MCB) sampler and prove that given access to the exact posterior marginals, it improves over the conditional-mean bridge used by
standard DDPM-style sampling.
Throughout this section, let $\mathcal X_0:=\{e_v:v\in\mathcal V\}^L$
denote the set of valid clean one-hot sequences.  We use the reverse-time grid
\[
    0=t_0<t_1<\cdots<t_K=T,
    \qquad
    \gamma_k:=t_{k+1}-t_k,
\]
and write
\[
    u_k:=T-t_k
\]
for the corresponding forward noise level.  Thus a reverse step from
\(t_k\) to \(t_{k+1}\) corresponds to a forward-noise transition from
\(u_k\) down to \(u_{k+1}=u_k-\gamma_k\).

For \(0\le s<t\), let
\[
    B_{s,t}(dz\mid y,x_0)
    :=
    \mathcal L(X_s\mid X_t=y,X_0=x_0)
\]
be the OU bridge transition kernel.  By \eqref{eq:bridge_marginal_clean}, we have
\begin{equation}
\label{eq:bridge_kernel_main}
    B_{s,t}(dz\mid y,x_0)
    =
    \cN
    \left(
        \frac{\sinh(t-s)}{\sinh t}x_0
        +
        \frac{\sinh s}{\sinh t}y,\,
        2\frac{\sinh s\,\sinh(t-s)}{\sinh t}I_D
    \right)(dz).
\end{equation}

\begin{assumption}[Exact posterior marginals]
\label{asmp:exact_marginals}
For every noise level \(t\), state \(x\), and token position \(\ell\), the FLM
returns the exact posterior token marginal
\[
    q_{t,\ell}(\cdot\mid x)
    :=
    \mathcal L(X_{0,\ell}\mid X_t=x).
\]
\end{assumption}

Under this assumption, the FLM defines the factorized endpoint posterior
\begin{equation}
\label{eq:factorized_endpoint_posterior}
    q_t^I(x_0\mid x)
    :=
    \prod_{\ell=1}^L
    q_{t,\ell}(x_{0,\ell}\mid x),
    \qquad
    x_0\in\mathcal X_0.
\end{equation}
This is exactly the posterior distribution that the FLM can represent using
token-wise marginals.  The central question is how to sample from the reverse
process induced by \eqref{eq:factorized_endpoint_posterior}.

\subsection{Marginal-conditioned bridge sampler}
\label{ref:sampling}

The exact posterior-predictive reverse kernel from noise level \(u_k\) to
\(u_{k+1}\) is
\begin{equation}
\label{eq:true_posterior_predictive_kernel}
    K_k^\star(dz\mid y)
    =
    \sum_{x_0\in\mathcal X_0}
    B_{u_{k+1},u_k}(dz\mid y,x_0)
    q(x_0\mid X_{u_k}=y).
\end{equation}
This kernel is generally intractable because the true joint posterior
\(q(x_0\mid X_{u_k}=y)\) couples all token positions.

The marginal-conditioned bridge sampler replaces the true joint endpoint
posterior by the factorized posterior available from the FLM:
\begin{equation}
\label{eq:mcb_kernel}
    K_k^{\MCB}(dz\mid y)
    =
    \sum_{x_0\in\mathcal X_0}
    B_{u_{k+1},u_k}(dz\mid y,x_0)
    q_{u_k}^I(x_0\mid y).
\end{equation}
Equivalently, at step \(k\), we sample
\[
    X_0^I\sim q_{u_k}^I(\cdot\mid y_{t_k})
\]
and then sample
\begin{equation}
\label{eq:mcb_transition_main}
    y_{t_{k+1}}\mid y_{t_k},X_0^I
    \sim
    \cN
    \left(
        \frac{\sinh(\gamma_k)}{\sinh(u_k)}X_0^I
        +
        \frac{\sinh(u_k-\gamma_k)}{\sinh(u_k)}y_{t_k},\,
        2\frac{\sinh(\gamma_k)\sinh(u_k-\gamma_k)}{\sinh(u_k)}I_D
    \right).
\end{equation}
In practice, we implement this sampler by substituting our learned approximation $p_{\theta, t, \ell}$ for $q^I_t$.

The DDPM conditional-mean bridge instead replaces the random endpoint
\(X_0^I\) by its mean
\[
    m_{u_k}(y)
    :=
    \E[X_0\mid X_{u_k}=y].
\]
Thus its one-step kernel is
\begin{equation}
\label{eq:ddpm_kernel}
    K_k^{\DDPM}(dz\mid y)
    =
    B_{u_{k+1},u_k}(dz\mid y,m_{u_k}(y)).
\end{equation}
The two samplers therefore use the same posterior marginals, but in different
ways: MCB samples from the endpoint distribution, while DDPM collapses that
distribution to its mean. We provide the pseudo-code of the algorithm in Appendix~\ref{app:pseudocode}.

\subsection{Theoretical Analysis}
\label{sec:theory}
In order to analyze our proposed sampler, we present the following results.
First, we identify the posterior-predictive reverse kernel naturally defined by an FLM.  
Second, we quantify the approximation made by replacing the true joint endpoint posterior by the product of its token marginals.  
Third, we compare the KL error of the continuous-time interpolations of the MCB sampler and the DDPM sampler using a Girsanov path-space KL identity with respect to the ground-truth process, where we show the KL error for MCB is bounded by the KL error for DDPM.

\begin{proposition}[MCB is the posterior-predictive sampler for the FLM posterior]
\label{prop:mcb_posterior_predictive}
Fix a reverse step \(k\) and current state \(y\).  The kernel
\(K_k^{\MCB}\) in \eqref{eq:mcb_kernel} is the exact posterior-predictive
bridge kernel induced by the factorized endpoint posterior
\(q_{u_k}^I(\cdot\mid y)\).  Moreover, \(K_k^{\MCB}\) and
\(K_k^{\DDPM}\) have the same conditional mean, but MCB preserves additional
endpoint variance:
\begin{align}
    \E_{\MCB}[Z\mid y]
    &=
    \E_{\DDPM}[Z\mid y],
    \label{eq:mcb_ddpm_same_mean}
    \\
    \Cov_{\MCB}(Z\mid y)
    &=
    \Cov_{\DDPM}(Z\mid y)
    +
    \left(\frac{\sinh(\gamma_k)}{\sinh(u_k)}\right)^2
    \Cov_{q_{u_k}^I}(X_0^I\mid y).
    \label{eq:mcb_extra_variance}
\end{align}
\end{proposition}
\begin{proof}
    We defer the proof to Appendix \ref{proof:mcb_posterior_predictive}.
\end{proof}

\begin{theorem}[The factorization error is conditional multi-information]
\label{thm:factorization_error}
Fix a noise level \(t\) and a state \(x\).  Let
\[
    q_t(x_0\mid x)
    :=
    \mathcal L(X_0=x_0\mid X_t=x)
\]
be the true joint clean posterior, and let \(q_t^I\) be the product of its token
marginals as in \eqref{eq:factorized_endpoint_posterior}.  Then
\begin{equation}
\label{eq:endpoint_factorization_error}
    \KL
    \left(
        q_t(\cdot\mid x)
        \,\middle\Vert\,
        q_t^I(\cdot\mid x)
    \right)
    =
    I(X_{0,1};\ldots;X_{0,L}\mid X_t=x),
\end{equation}
where the right-hand side is the conditional multi-information among clean token
positions.

Furthermore, for the one-step bridge kernels,
\begin{equation}
\label{eq:kernel_factorization_bound}
    \KL
    \left(
        K_k^\star(\cdot\mid y)
        \,\middle\Vert\,
        K_k^{\MCB}(\cdot\mid y)
    \right)
    \le
    I(X_{0,1};\ldots;X_{0,L}\mid X_{u_k}=y).
\end{equation}
\end{theorem}
\begin{proof}
    We defer the proof to Appendix \ref{proof:factorization_error}.
\end{proof}

\begin{theorem}[MCB improves the path-space KL over DDPM]
\label{thm:mcb_beats_ddpm}
Assume exact posterior marginals as in Assumption~\ref{asmp:exact_marginals}.
Let \(Y_{[0,T]}\) be the true reverse process.  Let
\(\hat Y^{\MCB}_{[0,T]}\) be the continuous-time process obtained by
interpolating Algorithm~\ref{alg:exact_mcb_ou} with OU bridges, and let
\(\hat Y^{\DDPM}_{[0,T]}\) be the corresponding frozen conditional-mean bridge
sampler.  Suppose both approximate samplers are initialized from the same
terminal law \(\pi_T\), and assume the standard absolute-continuity and Novikov
conditions under which the Girsanov identities below hold.  Then
\begin{equation}
\label{eq:main_kl_comparison}
    \KL
    \left(
        \mathrm{Law}(Y_{[0,T]})
        \,\middle\Vert\,
        \mathrm{Law}(\hat Y^{\MCB}_{[0,T]})
    \right)
    \le
    \KL
    \left(
        \mathrm{Law}(Y_{[0,T]})
        \,\middle\Vert\,
        \mathrm{Law}(\hat Y^{\DDPM}_{[0,T]})
    \right).
\end{equation}
Moreover, the difference is explicitly
\begin{align}
\label{eq:kl_gap_formula}
    &\KL
    \left(
        \mathrm{Law}(Y_{[0,T]})
        \,\middle\Vert\,
        \mathrm{Law}(\hat Y^{\DDPM}_{[0,T]})
    \right)
    -
    \KL
    \left(
        \mathrm{Law}(Y_{[0,T]})
        \,\middle\Vert\,
        \mathrm{Law}(\hat Y^{\MCB}_{[0,T]})
    \right)
    \nonumber\\
    &\qquad
    =
    \sum_{k=0}^{K-1}
    \int_{t_k}^{t_{k+1}}
    \frac{c_{T-t}^2}{\sigma_{T-t}^4}
    \E
    \left[
        \norm{
            m_{T-t}(X_{T-t})-m_{u_k}(X_{u_k})
        }^2
        -
        \norm{
            m_{T-t}(X_{T-t})-\bar m_{k,T-t}
        }^2
    \right]dt
    \ge 0.
\end{align}
Here
$
    \bar m_{k,u,\ell}
    :=
    \E[X_{0,\ell}\mid X_{u_k},X_{u,\ell}]
$ where $u=T-t$
is the endpoint estimate induced by the marginal-conditioned bridge after
marginalizing over the sampled endpoint.  The inequality is strict whenever, on
a set of positive measure, there exists a token position \(\ell\) such that
\[
    \E[X_{0,\ell}\mid X_{u_k},X_{u,\ell}]
    \neq
    \E[X_{0,\ell}\mid X_{u_k}].
\]
\end{theorem}
\begin{proof}
    We defer the proof to Appendix \ref{proof:mcb_beats_ddpm}.
\end{proof}
The key observation behind the proof is an $L^2$ projection argument.  On each discretization interval, DDPM approximates the current denoising target
\(\E[X_0\mid X_t]\) using only the coarser grid-point information
\(\E[X_0\mid X_{t_k}]\).  In contrast, the marginal-conditioned bridge refines this estimate blockwise by using
\(\E[X_{0,\ell}\mid X_{t_k},X_{t,\ell}]\), which conditions on strictly more information about the intermediate state.  Since conditional expectation is the optimal \(L^2\) predictor and conditioning on a larger sigma-algebra can only decrease the mean-squared error, the MCB endpoint estimate has no larger denoising error than the frozen DDPM estimate, with strict improvement whenever \(X_{t,\ell}\) contains additional information about \(X_{0,\ell}\) beyond \(X_{t_k}\).
\section{Experiments}
\label{sec:experiments}

We evaluate whether marginal-conditioned bridge sampling, additionally with the flexibility of  temperature scaling and nucleus sampling, empirically improves sample quality over standard ODE samplers for pretrained Flow Language Models.   

\paragraph{Model and dataset.}
We use the pretrained on LM1B~\cite{chelba2014billionwordbenchmarkmeasuring} and OWT~\cite{Gokaslan2019OpenWeb} FLM checkpoints from
\cite{lee2026flow}.  The checkpoints are trained on the datasets using one-hot token representations and a denoising
cross-entropy objective.  We keep the pretrained model fixed throughout all
experiments.  Thus, all differences in performance are entirely due to the sampling
procedure. Due to the limited space we only discuss the results evaluated on LM1B checkpoint. Results for the OWT checkpoint are qualitatively similar and are detailed in Appendix~\ref{apdx:OWT}.

\paragraph{Evaluation metrics.}
We follow the evaluation protocol of \cite{lee2026flow}.  For each sampler and
hyperparameter setting, we generate \(1024\) samples and compute generative
perplexity (Gen. PPL) using GPT-2 Large \citep{radford2019language}.  Since low Gen. PPL
can be misleading when a model produces repetitive text, we also report the
average per-sample unigram entropy.  A good sampler should have low Gen. PPL
while maintaining entropy close to the LM1B data entropy.  We use the same
sequence length, tokenizer, time grid, and model checkpoint as the original FLM
evaluation.

\paragraph{Baseline.}
The primary baseline is the standard ODE sampler used for FLM inference in
\cite{lee2026flow}.  This sampler integrates the probability-flow ODE using
Euler steps.  Let \(0=t_0<t_1<\cdots<t_K=1\) denote the sampling grid in the
FLM time parameterization (i.e. flow matching convention), and let \(y_{t_k}\) be the current continuous state.
Writing
\[
    \delta_\theta(y_{t_k},t_k)
    =
    \E_\theta[X_1\mid X_{t_k}=y_{t_k}]
\]
for the FLM denoiser, the Euler ODE update can be written as
\begin{equation}
\label{eq:ode_experiment_update}
    y_{t_{k+1}}^{\mathrm{ODE}}
    =
    \frac{1-t_{k+1}}{1-t_k}y_{t_k}
    +
    \frac{t_{k+1}-t_k}{1-t_k}
    \delta_\theta(y_{t_k},t_k).
\end{equation}
This is the conditional-mean bridge: the FLM posterior marginals are averaged
into a simplex-valued endpoint and the sampler bridges toward that mean.


\paragraph{Temperature and nucleus sampling.}
Because MCB samples from explicit token distributions, standard decoding
controls can be applied directly to the endpoint marginals.  Given logits
\(z_{\ell,v}\) or probabilities \(\pi_{\ell,v}\), temperature scaling with
temperature \(\tau>0\) is defined by
\begin{equation}
\label{eq:temperature_scaling_exp}
    \pi_{\ell,v}^{(\tau)}
    =
    \frac{\pi_{\ell,v}^{1/\tau}}
    {\sum_{v'\in\cV}\pi_{\ell,v'}^{1/\tau}}.
\end{equation}
Lower temperatures sharpen the endpoint posterior, while higher temperatures
increase endpoint diversity.  We also consider nucleus sampling with threshold
\(p\in(0,1]\).  For each position \(\ell\), we sort tokens by
\(\pi_{\ell,v}^{(\tau)}\), keep the smallest set \(S_{\ell, p}\) whose
cumulative probability is at least \(p\), set all other probabilities to
zero, and renormalize.  The default MCB setting is \(\tau=1\) and \(p=1\),
which uses the unmodified FLM posterior marginals.

\begin{figure}[ht]
  \centering
  
  \begin{minipage}[c]{0.4\textwidth}
    \centering
    \vspace{-0.2cm} 
    \includegraphics[width=1.0\linewidth]{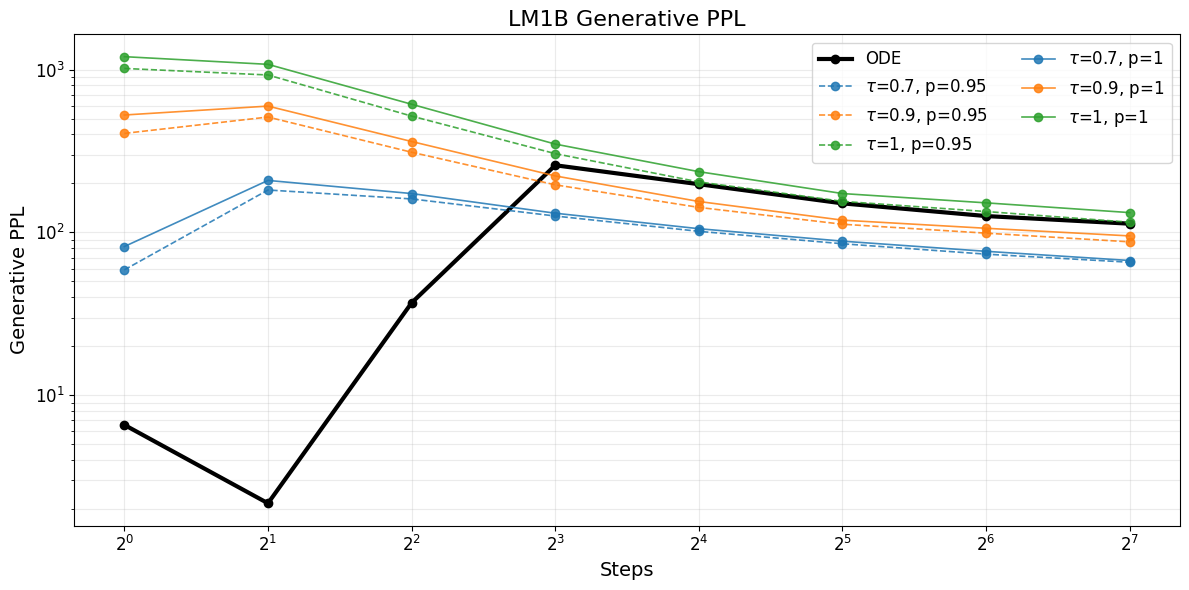}
    
    \vspace{-0.17cm} 
    
    \includegraphics[width=1.0\linewidth]{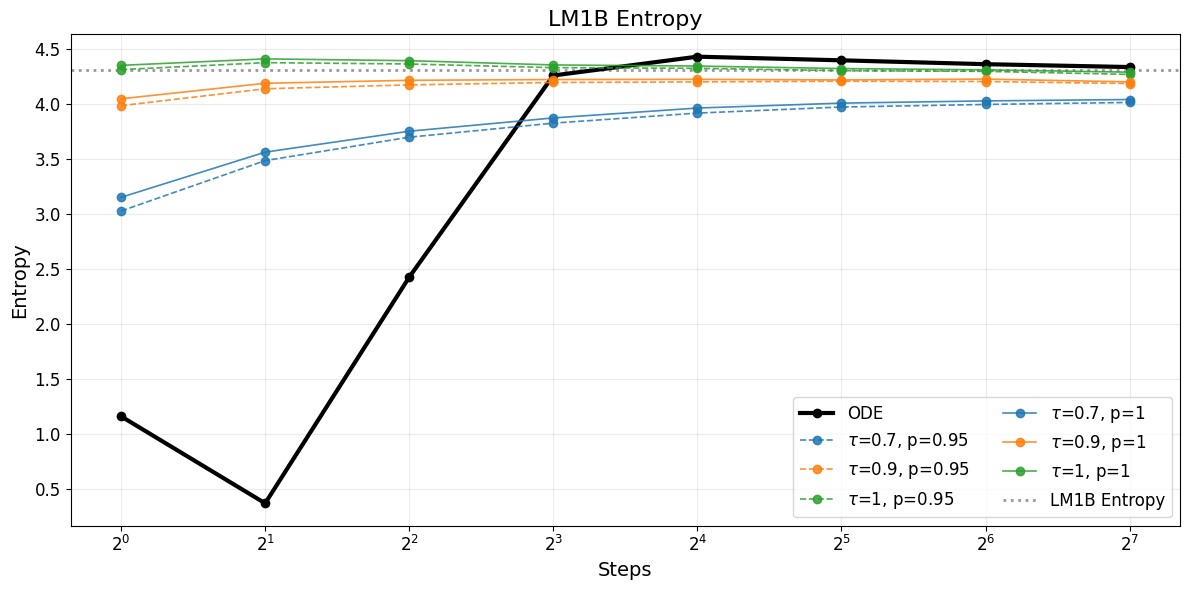}
  \end{minipage}
  \begin{minipage}[c]{0.585\textwidth}
    \centering
    \includegraphics[width=1.05\linewidth]{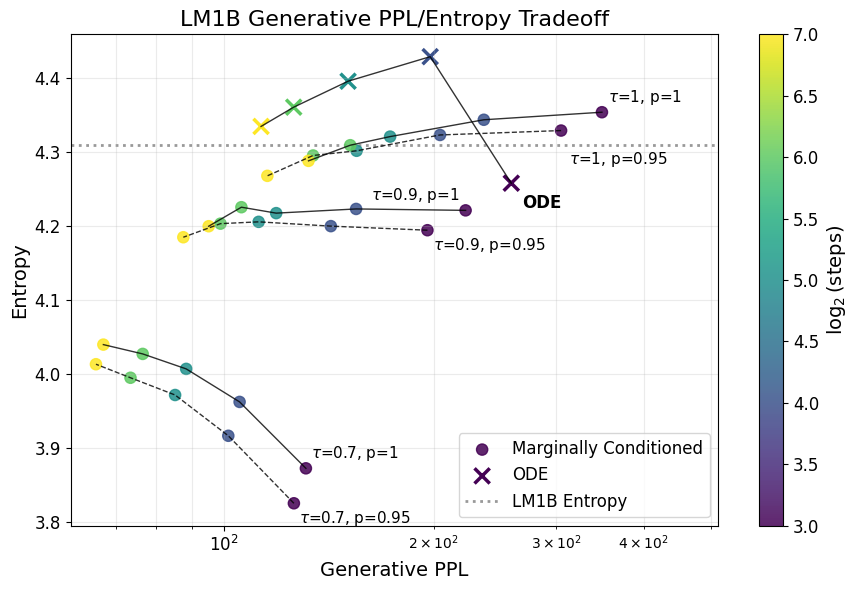}
  \end{minipage}
  
  \caption{Generative perplexity (left top) and entropy (left bottom) against the number of sampling steps for the standard ODE sampler and our MCB sampler with various configurations of temperature scaling $\tau$ and nucleus sampling $p$ on LM1B. The right plot shows the Generative PPL/Entropy Tradeoff. We note that the grey dotted line on the bottom-left plot shows the entropy of LM1B.}
  \label{fig:lm1b_mcb}
\end{figure}

\paragraph{Results.}

\looseness=-1
Figure~\ref{fig:lm1b_mcb} compares the standard ODE sampler with our
marginal-conditioned bridge sampler on LM1B with different settings for temperature scaling and nucleus sampling, where we plot generative perplexity and entropy as a function of the number of discretization steps in $\{2^i\}_{i=0}^7$ used.
We ran all of our experiments on Nvidia GH200 GPUs with each experimental configuration taking around one hour.

The top panel shows that the MCB sampler improves steadily in generative perplexity as the number of steps increases with a consistent
ordering with respect to temperature---sharper endpoint sampling
(\(\tau=0.7\)) gives the lowest generative perplexity among the marginal-conditioned variants, followed by \(\tau=0.9\) and then \(\tau=1\). Additionally, nucleus
sampling with \(p=0.95\) gives a small additional improvement over \(p=1\),
especially at small and intermediate step counts. The bottom panel shows the
corresponding entropy where we see while lower temperature reduces entropy, 
\(\tau=0.9\) and \(\tau=1.0\) remain close to the entropy of LM1B, giving a
better quality--diversity balance. Again, nucleus sampling with \(p=0.95\) provides a small change, decreasing entropy.


On the other hand, the generative perplexity for the ODE baseline dramatically drops for less than eight steps, but the corresponding entropy is
also extremely low, indicating degenerate low-diversity generations.
Additionally,
as the
number of ODE steps increases, entropy recovers toward the data entropy, but
generative perplexity becomes worse than the best marginal-conditioned variants.  In
contrast, the MCB sampler provides a smooth tradeoff: endpoint temperature controls
quality versus diversity, at sufficiently many steps the sampler remains
lower generative perplexity while maintaining non-collapsed entropy, and even in the low step regime, the MCB does not degenerate like the ODE sampler instead maintaining stable generative perplexity and entropy.

These results support the posterior-predictive interpretation of FLMs.  The FLM
output should not only be used as a simplex-valued conditional mean; it can be
used directly as a distribution over clean one-hot bridge endpoints.  This
training-free change gives better control over the quality--diversity tradeoff
while using the same pretrained model and the same number of model evaluations
as the standard ODE sampler.

\section{Conclusion and Limitations}
\label{sec:conclusion}
We have introduced a marginal-conditioned bridge sampler for Flow Language Models which exploits the discrete information already learned by FLMs---i.e. token-wise posterior  marginals. 
Rather than collapsing these marginals to a simplex-valued conditional mean, it samples clean data at the endpoint and uses the analytic OU bridge conditioned on that endpoint. The resulting method is training-free and requires no additional model evaluations; moreover, our sampler admits the incorporation of standard decoding techniques from language modeling such as temperature scaling and nucleus sampling, allowing for more flexibility in generation. 
Additionally, our analysis has shown that our marginal-conditioned bridge sampler has better performance in terms of KL when compared to standard samplers which instead utilize the conditional mean.

\paragraph{Limitations.}
While the marginal-conditioned bridge sampler yields a smaller discretization error and enables flexible decoding, it introduces structural challenges for inference acceleration via distillation. Standard flow map distillation relies on smooth, deterministic trajectories. In contrast, this approach injects discrete stochasticity at every step by sampling categorical proposals $x_0 \sim q_{u_k}^I(\cdot\mid y_{t_k})$, resulting in a non-smooth branching process.
\paragraph{Broader impacts.}
\label{sec:impact}
Because Flow Language Models and our proposed marginal-conditioned bridge sampler improve the efficiency and flexibility of text generation, they lower the computational barrier to deploying powerful language models. While this has positive implications for accessibility and efficient inference, it also inherits the dual-use risks common to all generative language models, such as the potential generation of toxic text, bias, or disinformation. Our work focuses on foundational algorithmic improvements and does not introduce new applications or specific mitigations for these broader societal risks.


 \section*{Acknowledgments}
LZ is supported by the EPSRC CDT in Modern Statistics and Statistical Machine Learning (EP/S023151/1).
IA was supported by the Engineering and Physical Sciences Research Council [grant number EP/T517811/1]. 

IA and LZ kindly thank Kianoosh Ashouritaklimi for his generous donation of one laptop.

\bibliographystyle{plain}
\bibliography{references}


\newpage

\appendix

\section{Pseudo-Code}
\label{app:pseudocode}
\begin{algorithm}[h!]
\caption{Marginal-conditioned OU bridge sampler}
\label{alg:exact_mcb_ou}
\begin{algorithmic}[1]
\Require Reverse grid \(0=t_0<t_1<\cdots<t_K=T\)
\Require Posterior-marginal model \(p_\theta\)
\Require Vocabulary \(\mathcal V\), sequence length \(L\), dimension \(D=L|\mathcal V|\)
\Require Optional decoding parameters: temperature \(\tau\), nucleus threshold \( p \)
\State Sample \(y_{t_0}\sim\mathcal N(0,I_D)\)
\For{\(k=0,\ldots,K-1\)}
    \State \(u_k\gets T-t_k\) \Comment{Current forward noise level}
    \State \(u_{k+1}\gets T-t_{k+1}\)
    \State \(\gamma_k\gets t_{k+1}-t_k=u_k-u_{k+1}\)
    \For{\(\ell=1,\ldots,L\)}
        \State Evaluate the posterior token marginal
        \[
            \pi_\ell(\cdot)
            \gets
            p_{\theta,u_k,\ell}(\cdot\mid y_{t_k})
        \]
        \State Optionally apply temperature scaling and/or nucleus sampling to \(\pi_\ell\)
        \State Sample \(w_\ell\sim\mathrm{Cat}(\pi_\ell)\)
        \State Set \((x_0)_\ell\gets e_{w_\ell}\)
    \EndFor
    \State Compute the exact OU bridge mean
    \[
        \mu_k
        \gets
        \frac{\sinh(u_{k+1})}{\sinh(u_k)}\,y_{t_k}
        +
        \frac{\sinh(u_k-u_{k+1})}{\sinh(u_k)}\,x_0
    \]
    \State Compute the exact OU bridge covariance
    \[
        \Sigma_k
        \gets
        2
        \frac{\sinh(u_{k+1})\sinh(u_k-u_{k+1})}{\sinh(u_k)}
        I_D
    \]
    \State Sample the next state from the pinned OU bridge
    \[
        y_{t_{k+1}}
        \sim
        \mathcal N(\mu_k,\Sigma_k)
    \]
\EndFor
\State \Return token decoding of \(y_{t_K}\)
\end{algorithmic}
\end{algorithm}

\section{Proofs}

\subsection{Proof of Proposition \ref{prop:mcb_posterior_predictive}}\label{proof:mcb_posterior_predictive}

\begin{proof}
\looseness=-1
The first statement follows directly from the posterior-predictive identity:
sample an endpoint \(X_0^I\sim q_{u_k}^I(\cdot\mid y)\), then sample the next
state from the bridge \(B_{u_{k+1},u_k}(\cdot\mid y,X_0^I)\).  This is exactly
\eqref{eq:mcb_kernel}.

Let
\[
    \beta_k:=\frac{\sinh(\gamma_k)}{\sinh(u_k)},
    \qquad
    \alpha_k:=\frac{\sinh(u_k-\gamma_k)}{\sinh(u_k)}.
\]
The MCB transition can be written as
\[
    Z=\alpha_k y+\beta_k X_0^I+\xi_k,
    \qquad
    \xi_k\sim
    \cN
    \left(
        0,\,
        2\frac{\sinh(\gamma_k)\sinh(u_k-\gamma_k)}{\sinh(u_k)}I_D
    \right),
\]
with \(\xi_k\) independent of \(X_0^I\).  Since
\[
    \E[X_0^I\mid y]
    =
    m_{u_k}(y),
\]
the conditional mean of \(Z\) is the same as the DDPM bridge-to-mean update.
Taking conditional covariance gives \eqref{eq:mcb_extra_variance}.  Therefore
DDPM is the moment-collapsed version of the factorized posterior-predictive
bridge.
\end{proof}

\subsection{Proof of Theorem \ref{thm:factorization_error}}\label{proof:factorization_error}

\begin{proof}
The identity \eqref{eq:endpoint_factorization_error} is exactly the definition
of conditional mutual-information:
\[
    I(X_{0,1};\ldots;X_{0,L}\mid X_t=x)
    =
    \KL
    \left(
        \mathcal L(X_0\mid X_t=x)
        \,\middle\Vert\,
        \prod_{\ell=1}^L
        \mathcal L(X_{0,\ell}\mid X_t=x)
    \right).
\]

For the bridge-kernel bound, observe that \(K_k^\star\) and \(K_k^{\MCB}\) are
obtained by applying the same blockwise OU bridge channel to two different
endpoint posteriors: the true joint posterior \(q(\cdot\mid X_{u_k}=y)\) and
the factorized posterior \(q_{u_k}^I(\cdot\mid y)\).  Since the OU bridge
transition is a Markov kernel, the data-processing inequality for KL gives
\[
    \KL
    \left(
        K_k^\star(\cdot\mid y)
        \,\middle\Vert\,
        K_k^{\MCB}(\cdot\mid y)
    \right)
    \le
    \KL
    \left(
        q(\cdot\mid X_{u_k}=y)
        \,\middle\Vert\,
        q_{u_k}^I(\cdot\mid y)
    \right).
\]
Applying \eqref{eq:endpoint_factorization_error} at \(t=u_k\) gives
\eqref{eq:kernel_factorization_bound}.
\end{proof}

\subsection{Proof of Theorem \ref{thm:mcb_beats_ddpm}}\label{proof:mcb_beats_ddpm}

\begin{proof}
We compare the two samplers through their continuous bridge interpolations.

\paragraph{Step 1: the true reverse drift.}
Let \(u=T-t\).  By Tweedie's formula and the OU bridge representation, the true
reverse drift can be written as
\begin{equation}
\label{eq:true_bridge_drift}
    b_t^\star(y)
    =
    \frac{m_u(y)-y\cosh u}{\sinh u},
    \qquad
    m_u(y):=\E[X_0\mid X_u=y].
\end{equation}

\paragraph{Step 2: the DDPM bridge drift.}
On the interval \([t_k,t_{k+1})\), DDPM freezes the endpoint estimate at the
grid point \(t_k\).  Its drift is therefore
\begin{equation}
\label{eq:ddpm_bridge_drift}
    b_t^{\DDPM}(y;y_{t_k})
    =
    \frac{m_{u_k}(y_{t_k})-y\cosh u}{\sinh u},
    \qquad
    u_k=T-t_k.
\end{equation}

\paragraph{Step 3: the MCB bridge drift.}
Conditioned on a sampled endpoint \(x_0\), the MCB interpolation is an OU bridge
with drift
\[
    \frac{x_0-y\cosh u}{\sinh u}.
\]
After marginalizing over the sampled endpoint, the drift is obtained by replacing
\(x_0\) by its conditional expectation given the current bridge state.  Thus
\begin{equation}
\label{eq:mcb_bridge_drift}
    b_t^{\MCB}(y;y_{t_k})
    =
    \frac{\bar m_{k,u}(y,y_{t_k})-y\cosh u}{\sinh u},
\end{equation}
where the \(\ell\)-th block is
\begin{equation}
\label{eq:mcb_endpoint_filter}
    \bar m_{k,u,\ell}(y,y_{t_k})
    =
    \E[X_{0,\ell}^I\mid X_{u_k}=y_{t_k},X_{u,\ell}=y_\ell]
\end{equation}
under the factorized endpoint posterior \(q_{u_k}^I(\cdot\mid y_{t_k})\) and the
coordinate-wise OU bridge likelihood.

Because the endpoint posterior marginals are exact, Bayes' rule identifies
\eqref{eq:mcb_endpoint_filter} with the corresponding true conditional
expectation:
\begin{equation}
\label{eq:mcb_filter_true_posterior}
    \bar m_{k,u,\ell}(y,y_{t_k})
    =
    \E[X_{0,\ell}\mid X_{u_k}=y_{t_k},X_{u,\ell}=y_\ell].
\end{equation}
Indeed, conditional on \(X_{u_k}=y_{t_k}\), the MCB prior for
\(X_{0,\ell}\) is the exact marginal posterior
\(q_{u_k,\ell}(\cdot\mid y_{t_k})\), and the OU bridge likelihood for
\(X_{u,\ell}=y_\ell\) depends only on \(X_{0,\ell}\) and
\((y_{t_k})_\ell\).

\paragraph{Step 4: Girsanov identities.}
The three processes have the same diffusion coefficient \(\sqrt{2}I_D\).  Under
the stated absolute-continuity assumptions, Girsanov's theorem gives
\begin{align}
\label{eq:girsanov_mcb_identity}
    &\KL
    \left(
        \mathrm{Law}(Y_{[0,T]})
        \,\middle\Vert\,
        \mathrm{Law}(\hat Y^{\MCB}_{[0,T]})
    \right)
    \nonumber\\
    &\qquad
    =
    \KL(p_T\Vert\pi_T)
    +
    \sum_{k=0}^{K-1}
    \int_{t_k}^{t_{k+1}}
    \frac{c_{T-t}^2}{\sigma_{T-t}^4}
    \E
    \left[
        \norm{
            m_{T-t}(X_{T-t})-\bar m_{k,T-t}
        }^2
    \right]dt,
\end{align}
and
\begin{align}
\label{eq:girsanov_ddpm_identity}
    &\KL
    \left(
        \mathrm{Law}(Y_{[0,T]})
        \,\middle\Vert\,
        \mathrm{Law}(\hat Y^{\DDPM}_{[0,T]})
    \right)
    \nonumber\\
    &\qquad
    =
    \KL(p_T\Vert\pi_T)
    +
    \sum_{k=0}^{K-1}
    \int_{t_k}^{t_{k+1}}
    \frac{c_{T-t}^2}{\sigma_{T-t}^4}
    \E
    \left[
        \norm{
            m_{T-t}(X_{T-t})-m_{u_k}(X_{u_k})
        }^2
    \right]dt.
\end{align}
The coefficient follows from
\[
    \frac{1}{4\sinh^2 u}
    =
    \frac{c_u^2}{\sigma_u^4},
    \qquad
    u=T-t,
\]
since \(1/\sinh u=2c_u/\sigma_u^2\).

\paragraph{Step 5: the denoising error of MCB is smaller.}
It remains to show that, for every \(k\) and \(u\in[u_{k+1},u_k]\),
\begin{equation}
\label{eq:projection_inequality_goal}
    \E
    \left[
        \norm{
            m_u(X_u)-\bar m_{k,u}
        }^2
    \right]
    \le
    \E
    \left[
        \norm{
            m_u(X_u)-m_{u_k}(X_{u_k})
        }^2
    \right].
\end{equation}

Fix a token position \(\ell\).  Define
\[
    M_\ell:=\E[X_{0,\ell}\mid X_u],
    \qquad
    G:=\sigma(X_{u_k}),
    \qquad
    H_\ell:=\sigma(X_{u_k},X_{u,\ell}).
\]
Since \(u<u_k\), the OU process satisfies the Markov relation
\[
    X_0 \longrightarrow X_u \longrightarrow X_{u_k}.
\]
Therefore
\[
    \E[X_{0,\ell}\mid G]
    =
    \E[M_\ell\mid G],
    \qquad
    \E[X_{0,\ell}\mid H_\ell]
    =
    \E[M_\ell\mid H_\ell].
\]
The first conditional expectation is the DDPM endpoint estimate for block
\(\ell\), and the second is the MCB endpoint estimate from
\eqref{eq:mcb_filter_true_posterior}.

Because \(G\subset H_\ell\), orthogonal projection in \(L^2\) gives
\[
    \E
    \left[
        \norm{
            M_\ell-\E[M_\ell\mid H_\ell]
        }^2
    \right]
    \le
    \E
    \left[
        \norm{
            M_\ell-\E[M_\ell\mid G]
        }^2
    \right].
\]
Substituting back, we get
\begin{align*}
    \E
    \left[
        \norm{
            \E[X_{0,\ell}\mid X_u]
            -
            \E[X_{0,\ell}\mid X_{u_k},X_{u,\ell}]
        }^2
    \right]
    \le
    \E
    \left[
        \norm{
            \E[X_{0,\ell}\mid X_u]
            -
            \E[X_{0,\ell}\mid X_{u_k}]
        }^2
    \right].
\end{align*}

Summing over \(\ell=1,\ldots,L\) gives
\eqref{eq:projection_inequality_goal}.

Combining \eqref{eq:girsanov_mcb_identity},
\eqref{eq:girsanov_ddpm_identity}, and
\eqref{eq:projection_inequality_goal} proves
\eqref{eq:main_kl_comparison}.  Subtracting the two Girsanov identities gives
\eqref{eq:kl_gap_formula}.  Strictness follows from strict improvement of the
\(L^2\) projection whenever
\[
    \E[X_{0,\ell}\mid X_{u_k},X_{u,\ell}]
    \neq
    \E[X_{0,\ell}\mid X_{u_k}]
\]
with positive probability for some \(\ell\).
\end{proof}

\section{Additional Experiments on OWT}
\label{apdx:OWT}
In this section, we present the evaluation of the marginal-conditioned bridge sampler on the OpenWebText (OWT) dataset. We follow the same experimental setup and evaluation protocol as the LM1B experiments detailed in Section~\ref{sec:experiments}, with two key distinctions: we generate longer text sequences of $1024$ tokens, and we evaluate the samplers over an extended range of discretization steps following the powers of two up to $1024$, specifically $\{2^i\}_{i=0}^{10}$.

As illustrated in Figure~\ref{fig:owt_mcb}, the OWT results exhibit similar fundamental dynamics to the LM1B evaluation while highlighting the scalability of the sampler over longer generations and more integration steps. The generative perplexity of the marginal-conditioned bridge sampler improves consistently as the step count increases. The temperature parameter provides a reliable mechanism for navigating the quality-diversity tradeoff: sharper sampling ($\tau=0.7$) achieves the lowest generative perplexity, while higher temperatures ($\tau=0.9$ and $\tau=1.0$) preserve higher sequence diversity, keeping the entropy closely aligned with the empirical entropy of the OWT dataset. Applying nucleus sampling at $p=0.95$ introduces a slight further reduction in both generative perplexity and entropy across the step variants.

The standard ODE baseline demonstrates severe degradation in the low-step regime (under eight steps), where a superficial drop in generative perplexity is paired with a near-zero entropy score, indicating repetitive, collapsed sequence generation. While the ODE sampler's entropy recovers as the step count scales toward $1024$, its generative perplexity remains strictly worse than the optimal marginal-conditioned configurations. The MCB sampler avoids this low-step collapse entirely, maintaining stable generative perplexity and entropy even at minimal step counts, and continues to leverage endpoint randomness to achieve superior language modeling performance at high step counts.
\begin{figure}[ht]
  \centering
  
  \begin{minipage}[c]{0.4\textwidth}
    \centering
    \vspace{-0.2cm} 
    \includegraphics[width=1.0\linewidth]{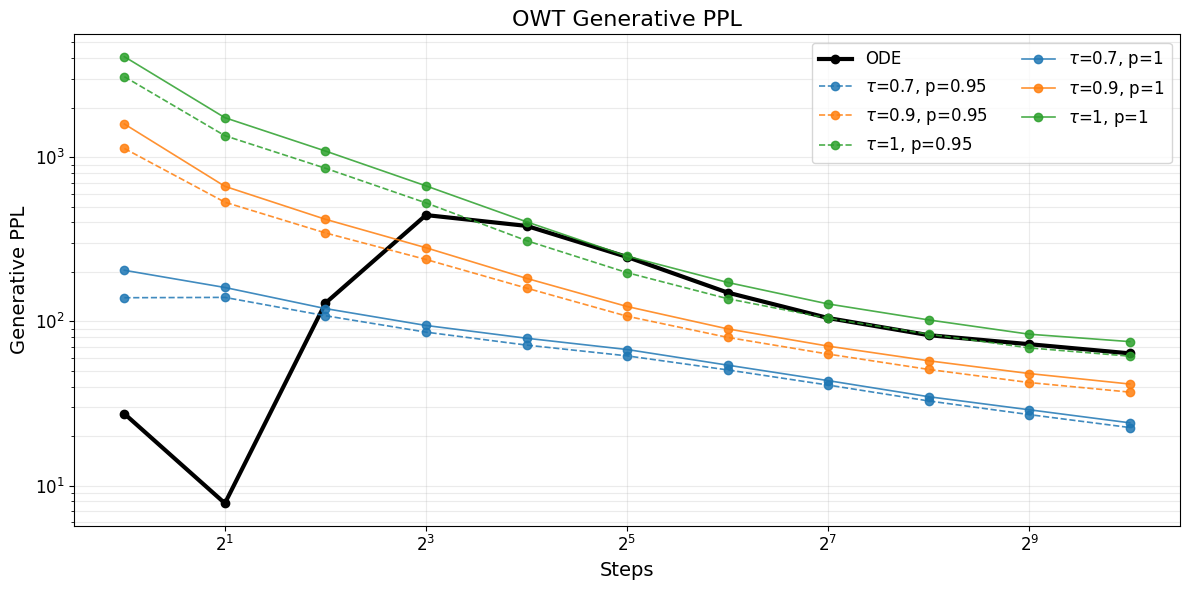}
    
    \vspace{-0.17cm} 
    
    \includegraphics[width=1.0\linewidth]{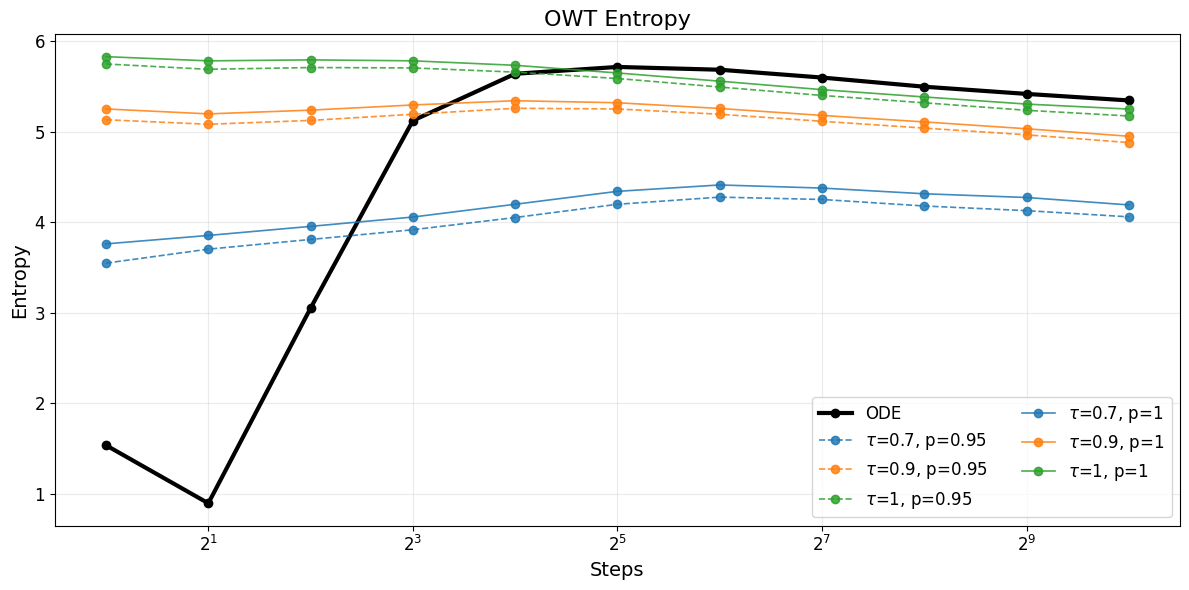}
  \end{minipage}
  \begin{minipage}[c]{0.585\textwidth}
    \centering
    \includegraphics[width=1.05\linewidth]{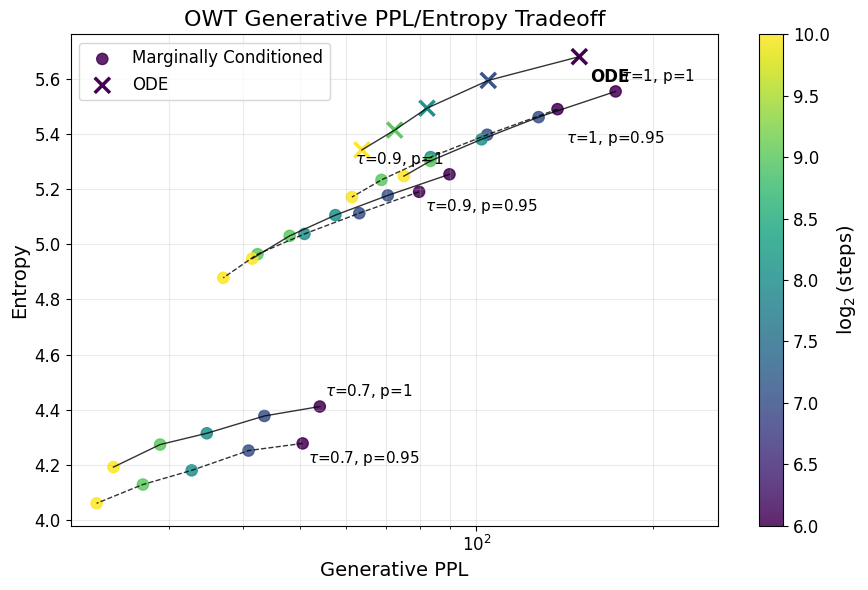}
  \end{minipage}
  
  \caption{Generative perplexity (left top) and entropy (left bottom) against the number of sampling steps for the standard ODE sampler and our MCB sampler with various configurations of temperature scaling $\tau$ and nucleus sampling $p$ on OWT. The right plot shows the Generative PPL/Entropy Tradeoff.}
  \label{fig:owt_mcb}
\end{figure}

\section{Samples}
\label{apdx:samples}
\begin{promptbox}{{OWT: 1024 steps, $T=0.7, p=0.95$}}
\small
<|endoftext|> it’s time, on the world’s stage, part of the day, as far as he wanted the government going, hearing from the person of that day.     \n\n     That on television couldn’t be a lot more important than his military. The day, though, and the military decision in the U.S. had not served his attention, and how well as the people of the people.     \n\n     Obama didn’t answer, and he didn’t have. So, he could hear the man, who was back at a microphone. “Go to the U,” and another said after the other.     \n\n   
\ommit
Now, as from a U.S. law and intelligence perspective, because of a lack of U.S. information about the kind in research, and the background of all of the children in there, this position could come to have the same, be given the recognition of the needs of, and the evidence against the children of some and the rest.     \n\n     In turn, of the British and the American, there’s not a difference of place, and only in the 8 U.S. House of the State.     \n\n     Now to the Ambassador’s that role in the building. In the U.S., the Department of the State Department of Defense played the same role as the military, which led the military on the 12 and 10. In the U.S. 12, the White House, which, President or Obama, could choose who was, was not “elected” when they were there<|endoftext|>
\end{promptbox}
\begin{promptbox}{{OWT: 1024 steps, $T=0.9, p=0.95$}}
\small
<|endoftext|> no national, and little corporate, time. They, attended by major cities, mass cities and commissions, were seen as the men and women’s worst bearer up until his time, through lines such as “In the streets, a look for the sea and air and out the desert. The that part of this town’s political arm because the state is the more you look in the eye the place we live gives you little, and the more new country you’re likely to.”  \n\n  Time for the People  \n\n  The picture of the new administration Zalin has been secretary of state in the last eight years.  \n\n  It seems crystal clear that the leader needs to find world-wide, loyalty to country and a power even more credicallyoming and carried on into modern life.
 \n\n 
\ommit
“We are not reneging on certain aspects of freedom, the history program has to be familiar with this and the experience of going to the American to, and we know that the government has to get into the list and align with it that we can do with the great Department of state as President and the Secretary of State. It is going to be a time for people to start and forget all the past achievements of the year.”  \n\n  The challenge for the present is power, and the Secretary says there is a chance that China will push the wrong way and follow its results, and this should not be the responsibility of its government.  \n\n  The Ambassador is precise in stating that, for 2017, state and the people of another are included. The local part of the board is given a state, federal, community, partnership.  \n\n  Although all of that year’s expected states is actually in place with this deal, there are the American, Chinese, European, and Chinese of state, 6,10,11, 19, 60,31, and Clinton, states.  \n\n  A great difference between these two<|endoftext|>
\end{promptbox}
\begin{promptbox}{{OWT: 1024 steps, $T=1.0, p=0.95$}}
\small
      <|endoftext|> your fingers with this key as you pedal them for the Olympicides, probably, as I most see them as over the past year. I said it, your frame? Either you have the wrong way off at 3:20 or I remember when it was me and the lead was at a 2 year road beginning on a halk in an instant and you will be left completelygin disappointed. It’s very easy. And it can pop up to shop many times for a successor. Pick it, buy a new one and it will last a day for the weary part, it’s just money. After four months the holds itself up in the time it takes to charge it. For the old son of it as a hand. It still costs so much – maybe one in a few.    \n\n    So watched the hisa line and his first ones champaxed so that was 4-8, no problem was in fact they just had one guy on. And a time 27:15 was all day. So the only penalty wish Miles was competitively contracting for a pick-up drop at the beginning was that at the top, the promoters start doing this, and just for the new mentor. But at this hour he heads up until the new horse jump off the top of that crack at his pace to get off the stage on the track. The new guy turns that into a top jump. I like the process but through it wrestling is so easy to be made, and most importantly it might be close.    \n\n    
      \ommit
      The Lord looking up one very young couple is interesting. And in our society, if they look better with less awe than at any. That was not when I wanted to know that whatever shows of the WhitneyPRO-Ike dances will continue when the National opens in the autumn of Public opens.    \n\n    It’s a funny thing. Phillip is Huntington, V.C. and he worked all day. “I’d just do songs by my side at the new vicar’s boarding school,” the 73-year-old told the New York Daily Company. “Let me famous. No TOO you my undergrad… It was afraid of my life.” And he adds that, By the end of this year the scouts had done so at about 20, and the damage’s going catching on in the wider world. Hollywoodirever, down in and live extremely hard – the psychology is extremely – if not very young; then he just means we’re showing more and more people that a new country is their generation going. Everything that<|endoftext|>,
\end{promptbox}
\begin{promptbox}{{OWT: 1024 steps, ODE sampler}}
    \small
    <|endoftext|> and none of it might be just hard to find out about the role ofand what it involves the politicians and their families, in a particular way of doing something that is ethical and political. Ask any of the kinds of questions that may be raised from the public interest from politicians and be very different and of different backgrounds and of course, from the law side, and it’s a wide-ranging list of questions to answer, but it is setting us up to throw us off. It’s second nature and we are trying to protect the public interest and by re-hroning over the details of the previous Party. We’re not always-telling the people about their concerns, but it’s certainly second, still the way to avoid the truth for some time and your only action from the legal department can help find a solution to it.” \n\n “At its core, the Drapers are letting go.”<|endoftext|>Arests with a more serious investigation were to be done about 10 p.m. to 4 a.m Monday-Friday. 
    \n\n
    \ommit
    Of course, in times of crisis, the forces that produced the Ethel “Frozen” goes along the same lines, although any sort of who has really never been seen, and those from the general bear the brunt of the doubt. \n\n Keenan, also then-chief of the BHA, came out in April calling for the murder of a police officer while co-acting out a murder investigation. It was published in the BSO’s “The Book of War,” including as a historical crime. Among the articles published in the BHA described the three women, she added, as “digngers and Ospers of the gang” that were composed of various police officers and more than one in the more serious cases. \n\n <|endoftext|>
\end{promptbox}

\end{document}